\documentclass{article}
\usepackage{amsmath,graphicx,mlspconf}
\usepackage{amsfonts}
\usepackage{amssymb}
\usepackage{xcolor}
\usepackage{subcaption}
\usepackage{algorithm}
\usepackage{array,multirow,graphicx}
\usepackage{multirow}
\usepackage{graphicx}
\usepackage{lscape}
\usepackage{amsthm}

\newtheorem{theorem}{Theorem}

\DeclareMathOperator{\sign}{sign}

\copyrightnotice{U.S.\ Government work not protected by U.S.\ copyright}

\copyrightnotice{978-1-6654-8547-0/22/\$31.00 {\copyright}2022 Crown}

\copyrightnotice{978-1-6654-8547-0/22/\$31.00
{\copyright}2022 European Union}

\copyrightnotice{978-1-6654-8547-0/22//\$31.00 {\copyright}2022 IEEE}

\toappear{2022 IEEE International Workshop on Machine Learning for Signal Processing, Agu.\ 22--25, 2022, Xi'an, China}

\title{BOSS: Bidirectional One-Shot Synthesis of Adversarial Examples}

\name{Anonymous}
\address{Anonymous}

\name{%
    Ismail R. Alkhouri$^{1}$%
    \qquad Alvaro Velasquez$^{2}$%
    \qquad George K. Atia$^{1,3}$\thanks{This work was supported in part by NSF CAREER Award CCF-1552497 and NSF Award CCF-2106339, DOE Award DE-EE0009152, AFRL Contract Number FA8750-20-3-1004, and AFOSR Award 20RICOR012.}%
}
\address{%
    $^{1}$ \small{Department of Electrical and Computer Engineering,
University of Central Florida, Orlando FL, USA 
}  
\\%
    $^{2}$ \small{Information Directorate, Air Force Research Laboratory, Rome NY, USA 
    } 
    \\
    $^{3}$ \small{Department of Computer Science, University of Central Florida, Orlando FL, USA}%
}

\begin{document}

\maketitle

\begin{abstract}

The design of additive perturbations to the inputs of classifiers has become a central focus of adversarial machine learning. An alternative approach is to synthesize adversarial examples using structures akin to generative adversarial networks, albeit with the use of large amounts of training data. By contrast, this paper considers the one-shot synthesis of adversarial examples that requires only a single reference datum. In particular, we explore solutions where the generated data must simultaneously satisfy user-defined constraints on its structural similarity to the reference input datum and the output of the classifier of interest it induces. This gives rise to what we call the Bidirectional One-Shot Synthesis (BOSS) problem. We prove that the BOSS problem is \textbf{NP-complete}. The experimental results verify that the targeted and confidence reduction attack methods developed either outperform or on par with state-of-the-art methods.

\end{abstract}
%
\section{Introduction}
\label{sec:intro}

The problem of robustness is being assessed in adversarial machine learning via additive perturbations to data and the synthesis of adversarial examples, which are often used to test the robustness of a given model. In this paper, we reconcile the notion of one-shot learning \cite{wang2020generalizing} and the synthesis of adversarial examples for the first time in what we call one-shot synthesis. In particular, given a datum $\mathbf{x}_{\textrm{d}}$ and a pre-trained model $p(.\:;\theta)$ parameterized by $\theta$, we propose a synthesis procedure that generates a new datum $\mathbf{x}$ to be used as an input to $p(.\:;\theta)$ such that constraints are satisfied on both the input structure and the output inference. In terms of the input, we ensure that $\mathbf{x}$ is similar to the given reference datum $\mathbf{x}_{\textrm{d}}$ by enforcing a small distance $d(\mathbf{x}, \mathbf{x}_{\textrm{d}}) \leq \delta_s$. In terms of the output, we generate $\mathbf{x}$ such that it approximately induces a user-defined output distribution $p_{\textrm{d}}$ as the inference result $p(.\:;\theta)$ of the pre-trained model by enforcing a small distance $D(p(\mathbf{x}; \theta), p_{\textrm{d}}) \leq \delta_c$. In this sense, the underlying Bidirectional One-Shot Synthesis (BOSS) problem is concerned with generating data satisfying constraints on both the input and output directions of the given classifier $p(.\:; \theta)$. \textcolor{black}{By controlling the induced output distribution, our approach generalizes traditional notions of targeted and non-targeted attacks \cite{machado2020adversarial}.} Confidence reduction attacks can also be implemented in our approach, where the goal is to lower the confidence level of the true label to cause ambiguity \cite{papernot2016limitations}, specifically against systems for which a confidence threshold is introduced and the classification is only regarded if the prediction confidence score is above that threshold \cite{pmlr-v119-stutz20a}.

We propose a solution to the BOSS problem by leveraging generative models whose parameters are updated based on the distance between the given datum $\mathbf{x}_{\textrm{d}}$ (distribution $p_{\textrm{d}}$) and the synthesized input datum $\mathbf{x}$ (output inference $p(\mathbf{x}\:;\theta)$). See Figure \ref{fig:motiv} for the problem description and BOSS samples. 

\begin{figure*}[t]
    \centering
    \includegraphics[width=18cm]{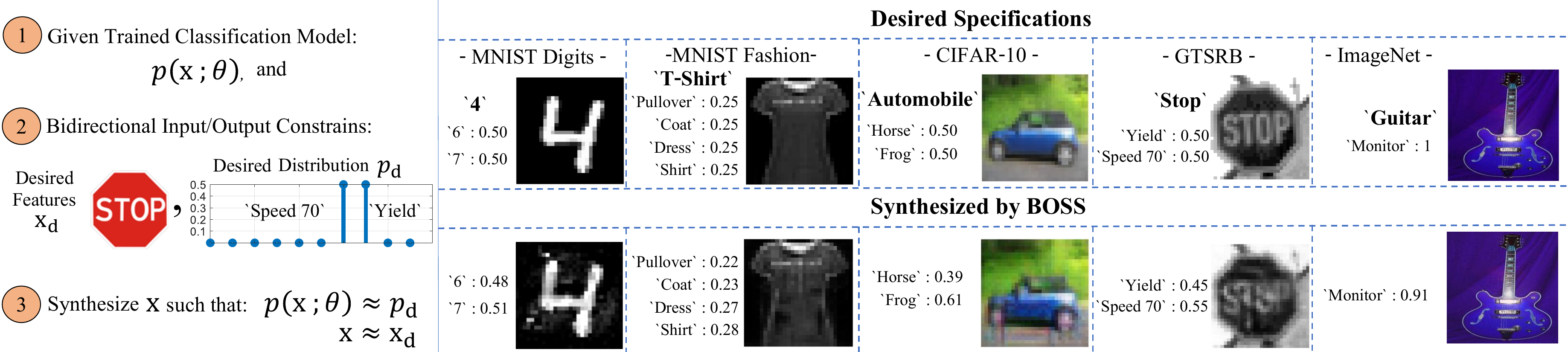}
    \caption{\small{Problem demonstration (\textit{left}) and examples (\textit{right}). The true labels are placed on the left of each desired features (image) with bold font. First row of images represent the desired features where the desired PMFs are placed on the left of each sample. The second row presents the synthesized examples, by BOSS, with their corresponding predictions w.r.t their trained classifiers. From left to right, samples are picked from the MNIST digits \cite{lecun2010mnist}, MNIST fashion \cite{DBLP:journals/corr/abs-1708-07747}, CIFAR-10 \cite{krizhevsky2009learning}, GTSRB \cite{Stallkamp2012}, and ImageNet \cite{5206848}, respectively.  
    \vspace{-0.5cm}
    }}
    \label{fig:motiv}
\end{figure*}

It is worth noting that our generative approach is a one-shot synthesis solution in that it only requires a single datum $\mathbf{x}_{\textrm{d}}$, which mitigates the excessive data requirements of popular methods based on Generative Adversarial Networks (GANs) \cite{goodfellow2014generative}. In fact, our proposed framework is more similar to the additive attack methods where a large body of works are presented such as the CW attack \cite{carlini2017towards}, the 
L-BFGS attack \cite{szegedy2013intriguing}, Deepfool \cite{moosavi2016deepfool}, Fast Adaptive Boundary (FAB) attack \cite{croce2020minimally}, saliency map attack \cite{papernot2016limitations}, and NewtonFool \cite{jang2017objective}. 

The contributions of the paper are the following. First, we present the BOSS problem to synthesize feature vectors that follow some desired input and output specifications. Second, we prove that BOSS is \textbf{NP-complete}. Our third contribution is the proposed algorithmic procedure that is based on generative networks and the back-propagation algorithm \cite{riedmiller1993direct} to produce (from scratch) these examples in a white-box settings. Fourth, we present methods to select the input/output specifications to generate targeted adversarial attacks, confidence reduction attacks, and decision boundary samples. On different attack evaluation metrics, we show that BOSS either on par or outperform state-of-the-art methods. Further, we show samples from small-scale and large-scale datasets on famous state-of-the-art classification architectures.


\section{Problem Formulation \& Characterization}
\label{Problem Formulation}

Suppose we have some trained model $p$ with parameters $\theta$ (e.g., a trained Neural Network) and a probability distribution $p(.\:; \theta) : \mathbb{R}^N \rightarrow \Delta^M$ over the output of the model with entries $p_{m}(\mathbf{x}\:;\theta)$ for $m\in[M]:=\{1,2,\dots,M\}$, where $M$ is the total number of outputs, and $\Delta^M$ is the probability simplex over $M$ dimensions.



Given a clean example (desired input features) $\mathbf{x}_{\textrm{d}}$, the well-known formulation of the basic iterative extension of the Fast Gradient Sign Method (FGSM) method \cite{kurakin2016adversarial} generates an adversarial example $\mathbf{x}$ by minimizing some differentialable loss function between $p(\mathbf{x};\theta)$ and $p_{\textrm{d}}$. The distance between 
$\mathbf{x}$ and $\mathbf{x}_{\textrm{d}}$, however, is restricted to the $l_p$ norm. A more general formulation is used in \cite{carlini2017towards} where the loss functions on the input and output of the classifier of interest can be chosen more flexibly. Therefore, we will compare our approach to the attacks in \cite{carlini2017towards} and an advanced version in \cite{chen2018ead}.

Let $d : \mathbb{R}^N \times \mathbb{R}^N \rightarrow [0,1]$ and $D:\Delta^M \times \Delta^M \rightarrow [0,1]$ denote distance functions between two feature vectors and distributions, respectively, where a value $0$ indicates identical arguments. 



\newtheorem{definition}{Definition}
\begin{definition}[BOSS Problem]
Given a learning model $p(.\:; \theta): \mathbb{R}^N \rightarrow \Delta^M$ parameterized by $\theta$, a tensor $\mathbf{x}_\textrm{d} \in \mathbb{R}^N$, and a probability distribution $p_\textrm{d} \in \Delta^M$, find an input tensor $\mathbf{x} \in \mathbb{R}^N$ such that $d(\mathbf{x}, \mathbf{x}_\textrm{d}) \leq \delta_\textrm{s}$ and $D(p(.\:; \theta), p_\textrm{d}) \leq \delta_\textrm{c}$, where upper bounds $\delta_\textrm{s}$ and $\delta_\textrm{c}$ and loss functions $d$ and $D$ are given.
\label{def:BOSS}
\end{definition}

First, we prove that the BOSS problem is \textbf{NP-complete}. This establishes that, in general, there is no polynomial-time solution to the BOSS problem unless \textbf{P} = \textbf{NP}. In Section \ref{sec: BOSS}, we develop a generative approach to obtain an approximate solution to the BOSS problem. 
\begin{definition}[CLIQUE Problem]
Given an undirected graph $G = (U, E)$ and an integer $k$, find a fully connected sub-graph induced by $U' \subseteq U$ such that $|U'| = k$.
\label{def:clique}
\end{definition}
\begin{theorem}
The Bidirectional One-Shot Synthesis (BOSS) problem in Definition \ref{def:BOSS} is \textbf{NP-complete}.
\label{theorem:complexity}
\end{theorem}
\begin{proof}
It is easy to verify that the BOSS problem is in \textbf{NP} since, given a tensor $\mathbf{x}$, one can check whether the input and output constraints $d(\mathbf{x}, \mathbf{x}_\textrm{d}) \leq \delta_\textrm{s}$ and $D(p(\mathbf{x}\:;\theta), p_\textrm{d}) \leq \delta_\textrm{c}$ are satisfied in polynomial time. It remains to be shown whether the BOSS problem is \textbf{NP-hard}. We will establish this result via a reduction from the CLIQUE problem in Definition \ref{def:clique}. Given a CLIQUE instance $\left< G = (U, E), k \right>$ with $|U| = n$ and $|E| = m$, we construct its corresponding BOSS instance $\left< p(.\:; \theta), \mathbf{x}_\textrm{d}, p_\textrm{d}, \delta_\textrm{s}, \delta_\textrm{c} \right>$ as follows. Let $\mathbf{x}_\textrm{d} = \mathbf{0}$ denote the all-zeroes vector and let $p_\textrm{d}$ be defined as the desired output distribution below
\begin{equation}
    p_\textrm{d} = \left( \frac{e^{\epsilon k}}{e^{\epsilon k (k - 1) / 2} + e^{\epsilon k}}\:,\: \frac{e^{\epsilon k (k - 1) / 2}}{e^{\epsilon k (k - 1) / 2} + e^{\epsilon k}} \right)^T,
\end{equation}
where $\epsilon \leq 1 - \sqrt{1 - (1 / (k + 1))}$. Finally, let $\delta_\textrm{s} = k / n$ and $\delta_\textrm{c} = 0$. We choose the mean square error loss (MSE) function to compute $d(\mathbf{x}, \mathbf{x}_\textrm{d}) \leq \delta_\textrm{s}$. The choice of loss function for computing $D(p(\mathbf{x}; \theta), p_\textrm{d}) \leq \delta_\textrm{c}$ is superfluous since we have chosen $\delta_\textrm{c} = 0$. For the given trained model $p(.\:; \theta)$, we define its connectivity and parameters $\theta$ as follows. The input layer consists of $n$ entries given by the solution $\mathbf{x} \in [0, 1]^n$. There is one hidden layer consisting of $n + m$ ReLU functions $\sigma_1, \dots, \sigma_{n + m}$ such that the first $n$ ReLU functions have a bias term of $\epsilon - 1$ and the next $m$ ReLU functions have a bias term of $\epsilon - 2$. Finally, there is an output layer with two softmax output activation functions $p_1(\mathbf{x}\:; \theta)$ and $p_2(\mathbf{x}\:;\theta)$. Let $\theta^h_{ij}$ denote the weight of the connection between the $i^\text{th}$ input $x_i$ and the $j^\text{th}$ ReLU activation $\sigma_j$ in the hidden layer. For each $u_i \in U$ in the given CLIQUE instance, we have $\theta^h_{ii} = 1$. The outputs of these $n$ ReLU activation functions are fully connected to the softmax output activation function $p_1(\mathbf{x}\:; \theta)$, each with a corresponding weight of $1$. For each edge $e_k = (u_i, u_j) \in E$, we have $\theta^h_{i, n + k} = \theta^h_{j, n + k} = 1$. This defines the input connectivity of ReLU functions $\sigma_{n + 1}$ to $\sigma_{n + m}$. The outputs of these are then fully connected to the second softmax function $p_2(\mathbf{x}\:;\theta)$, each with weight $1$. See Figure \ref{fig:exampleReduction} for an example. We now prove that there is a clique of size $k$ in $G$ if and only if there is a feasible solution $\mathbf{x}$ to the reduced BOSS instance. \par

($\implies$) Assume there is a clique of size $k$ in $G$. We can derive a feasible solution $\mathbf{x}$ to the reduced BOSS instance as follows. For every vertex $u_i \in U$ in the clique, let $x_i = 1$ and let all other values of $\mathbf{x}$ be $0$. The corresponding MSE loss is $d(\mathbf{x}, \mathbf{x}_\textrm{d}) = k / n$, thereby satisfying the input constraint defined by $\delta_\textrm{s}$. The solution $\mathbf{x}$ induces an output of $\sigma_i(x_i + (\epsilon - 1)) = \epsilon$ for each entry of $\mathbf{x}$ corresponding to a vertex $u_i$ in the clique and an output of $0$ for all other entries. Thus, we have $k$ inputs of value $\epsilon$ into the first softmax output function. Now, let us consider the edges induced by this clique. For each edge $e_k = (u_i, u_j) \in E$ in the clique, we have $\sigma_{n + k} (x_i + x_j + (\epsilon - 2)) = \epsilon$ and an output of $0$ for all other edges. Since there are $k (k - 1) / 2$ edges in a clique of size $k$, this yields $k (k - 1) / 2$ inputs of value $\epsilon$ into the second softmax output function. Thus, we have $p(\mathbf{x}\:;\theta) = p_\textrm{d}$ and the constraint $D(p(\mathbf{x}; \theta), p_\textrm{d}) \leq 0 = \delta_\textrm{c}$ is satisfied.  As a caveat, it is worth noting that, under this construction, the all-zeroes vector $\mathbf{0}$ yields equal outputs $p_1(\mathbf{0}; \theta) = p_2(\mathbf{0}; \theta) = 1/2$. Per the preceding arguments, it is also the case that a feasible solution $\mathbf{x}$ derived for a clique of size $k = 3$ will output $p_1(\mathbf{x}; \theta) = p_2(\mathbf{x}; \theta) = 1/2$. This is because $k = k (k - 1) / 2$ for $k = 3$. Thus, for the remainder of the proof, we assume that cliques of interest are of size $k > 3$.

\begin{figure}[t]
\vspace{-0.2cm}
    \centering
    \includegraphics[width=9.0cm]{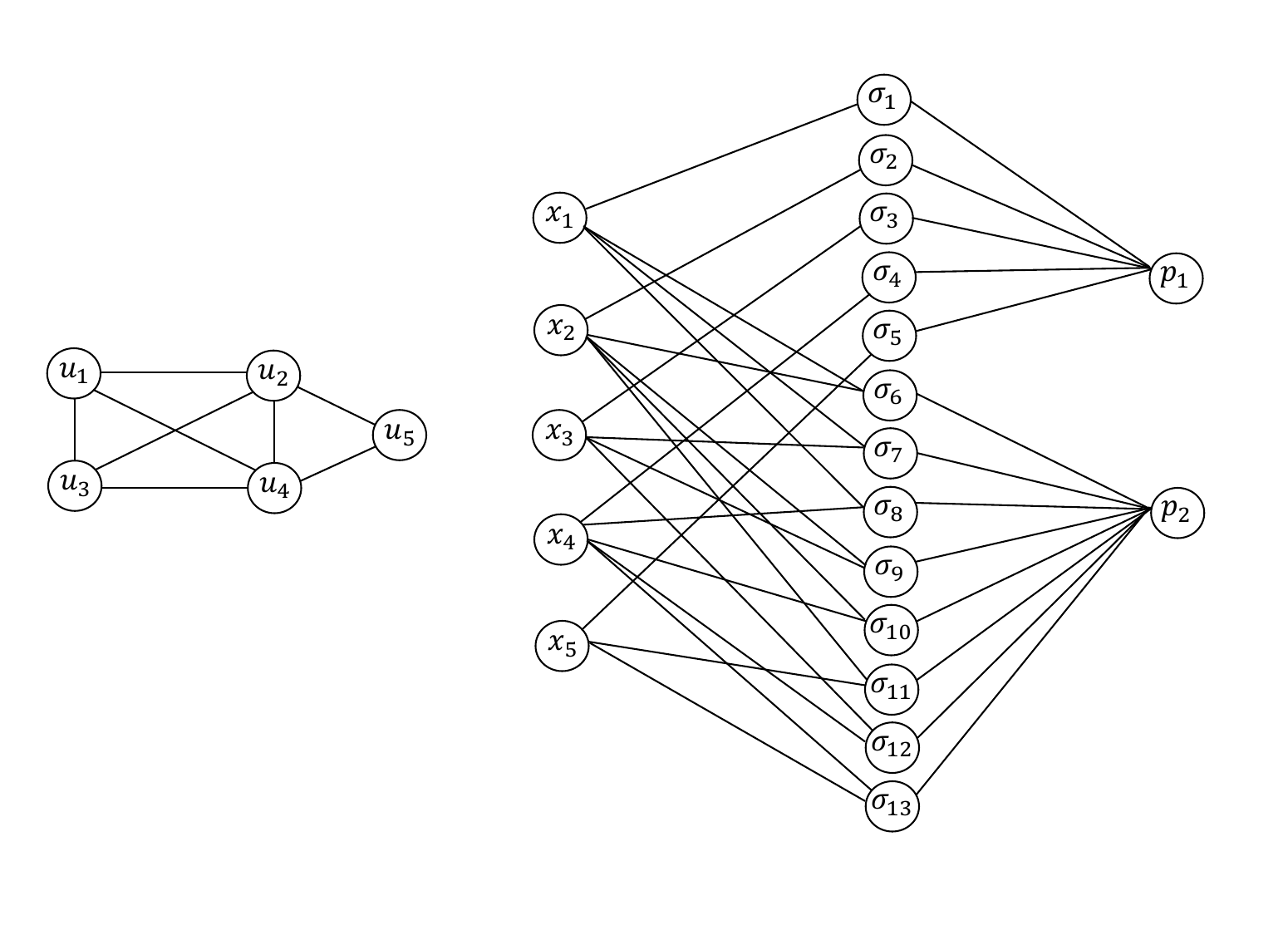}
    \vspace{-1cm}
    \caption{\small{Example reduction from a graph $G$ (\textit{left}) to a classifier $p(.\:; \theta)$ (\textit{right}).}}
    \vspace{-0.5cm}
    \label{fig:exampleReduction}
\end{figure}

($\impliedby$) We prove the contrapositive. That is, if there is no clique of size $k$ in $G$, then the reduced BOSS instance is infeasible. We proceed by showing that there must be exactly $k$ non-zero entries in $\mathbf{x}$ in order to satisfy constraints $d(\mathbf{x}, \mathbf{x}_\textrm{d}) \leq k / n$ and $D(p(\mathbf{x}\:;\theta), p_\textrm{d}) \leq 0$ and that, if there is no clique of size $k$, then there is no choice of $k$ non-zero entries in $\mathbf{x}$ that will satisfy $D(p(\mathbf{x}\:;\theta), p_\textrm{d}) \leq 0$. Note that there must be at least $k$ entries in $\mathbf{x}$ with value strictly greater than $(1 - \epsilon)$ in order to yield an input of $\epsilon k$ into the first softmax output function and satisfy the first entry in $p_\textrm{d}$. Let us consider the minimum MSE loss for a solution $\mathbf{x}$ with more than $k$ non-zero entries. For $k + 1$ entries of value strictly greater than $(1 - \epsilon)$, we have $d(\mathbf{x}, \mathbf{x}_\textrm{d}) = (k + 1)(1 - \epsilon)^2 / n$. With some algebraic manipulation, we have that, for any value of $\epsilon \leq 1 - \sqrt{1 - (1 / (k + 1))}$, $(k + 1)(1 - \epsilon)^2 / n > k / n$, thereby violating the constraint $d(\mathbf{x}, \mathbf{x}_\textrm{d}) \leq \delta_\textrm{s}$. Thus, there must be exactly $k$ non-zero entries in $\mathbf{x}$. Now, let us consider the second softmax output function, which requires an input of $\epsilon k (k - 1) / 2$. Since there is no clique of size $k$ in $G$, any choice of $k$ vertices in $G$ will induce a set of edges whose cardinality is strictly less than $k (k - 1) / 2$. Therefore, the output of the second softmax function will be strictly less than the second entry in $p_\textrm{d}$. This violates the constraint $D(p(\mathbf{x}\:; \theta), p_\textrm{d}) \leq 0$.
\end{proof}

Note that, for a given CLIQUE instance in the proof of Theorem \ref{theorem:complexity}, the corresponding reduced BOSS instance is such that, if there exists a polynomial-time solution to the BOSS problem, then we could use this solution to solve the CLIQUE problem in polynomial time. This would imply that \textbf{P} = \textbf{NP}. We therefore conjecture that a polynomial-time solution to the BOSS problem is not likely to exist.


\section{Generative Approach}
\label{sec: BOSS}

To obtain a solution to the BOSS problem in Definition \ref{def:BOSS}, we take a generative approach in which $\mathbf{x}$ is obtained as the output of a generative network, $g(.\:;\phi):\mathbb{R}^Q \rightarrow \mathbb{R}^N$, with parameters $\phi$, i.e., $g(\mathbf{z}\:;\phi) = \mathbf{x}$,
\textcolor{black}{where $\mathbf{z} \in\mathbb{R}^Q$ is a random input to the generative network}. We utilize the adjustable parameters of network $g$ for the 
objectives of BOSS. Therefore, we define the combined network $h(.\:;\psi): \mathbb{R}^Q \rightarrow [M]$, whose layers are the concatenation of the layers of $g$ and $p$, where $\psi=\{\phi,\theta\}$. In other words, $h(\mathbf{z}\:;\psi) = p(\mathbf{x}\:;\theta) = p(g(\mathbf{z}\:;\phi)\:;\theta)$.
%
We augment a repeated version of vector $\mathbf{z}$ to create a small training dataset. 
Given the two objectives of BOSS, and the utilization of the adjustable parameters of network $h$, $\phi$, we introduce the surrogate losses $\mathcal{L}_h(p(g(\mathbf{z}\:;\phi)\:;{\theta}),p_{\textrm{d}})$ and $\mathcal{L}_g(g(\mathbf{z}\:;\phi), \mathbf{x}_{\textrm{d}})$, 
and use the back-propagation algorithm \cite{riedmiller1993direct} to optimize $\phi$ based on the minimization
%
%
%
\begin{equation}
\begin{gathered} \label{eqn: main wrt phi}
\min_{\phi} \Big[ \mathcal{L}_g \Big(g(\mathbf{z}\:;\phi),\mathbf{x}_{\textrm{d}}\Big) +
\lambda \mathcal{L}_h \Big(p(g(\mathbf{z}\:;\phi)\:;\theta),p_\textrm{d}\Big) \Big],
\end{gathered} 
\end{equation}
where $\lambda$ is a loss weight. It is important to note that \eqref{eqn: main wrt phi} is used to update parameters $\phi$ while the trained classifier parameters $\theta$ remain unchanged. Due to the use of network $h$, the surrogate loss functions $\mathcal{L}_g$ and $\mathcal{L}_h$ can be selected as the \textcolor{black}{MSE} and the categorical cross-entropy loss, respectively.\par

%
In the following, we present an algorithmic approach to solve BOSS by iteratively optimizing \eqref{eqn: main wrt phi}.
At every iteration, the adjustable parameters $\phi$ of the generator model $g$ are updated to satisfy the two objectives of small PMF distance from $p_\textrm{d}$ and high similarity of the generated example to $\mathbf{x}_{\mathrm{d}}$. We define an exit criteria if either a maximum number of iterations/steps is reached, or if a feasible solution per Definition \ref{def:BOSS} is found given $\mathbf{x}_{\textrm{d}}$, $p_{\textrm{d}}$, $\delta_{\textrm{s}}$, and $\delta_{\textrm{c}}$.
\par
The parameter $\lambda$ in \eqref{eqn: main wrt phi} weighs the relative importance of each loss function to both avoid over-fitting and handle situations in which the solver converges for one loss function prior to the other \cite{goodfellow2016deep}. 
We propose a dynamic update that depends on the distance $D$ 
between the desired specification $p_{\textrm{d}}$ and the actual output $p(g(\mathbf{z}\:;\phi)\:;\theta)$ at every iteration. Specifically, we update $\lambda$ as
\begin{equation} \label{eqn: lablda h update}
\begin{gathered}
\lambda \leftarrow \sigma\left(\lambda - \lambda^{0} \frac{\delta_\textrm{c}}{D} \sign\Big(\frac{\delta_{\textrm{c}}}{D} -1\Big)\right).
\end{gathered} 
\end{equation}
As such, it is required to have the distance function $D$ returning values in the range of $[0,1]$. Here, we utilize the Jensen-Shannon (JS) divergence distance \cite{lin1991divergence}, which returns $0$ for two equivalent PMFs and is upper bounded by $1$. The updates are also a function of the initial selection of $\lambda$ denoted $\lambda^{0}$.
In this paper we focus on the task of image classification. 
\textcolor{black}{The authors in \cite{laidlaw2020perceptual} proposed the LPIPS distance metric as a measure of similarity that mimicks human perceptibility. However, since this metric is classifier-dependent, here we use a universal metric}. 
Specifically, we set $d = 1-I$, where $I$ is the Structural Similarity Index (SSIM) \cite{wang2004image}, which is equal to $1$ for two identical images and captures luminance, contrast, and structure in the measurements.  
%

The signum function $\sign(.)$ is used to determine whether to increase or decrease $\lambda$, based on the ratio of the actual and desired specifications which regulates the amount of change. The ReLU function $\sigma(.\:)$ prevents $\lambda$ from becoming negative.  
This occurs when the desired $p_{\textrm{d}}$ is easily attained in early steps of the algorithm. The procedure is presented in Algorithm \ref{alg:confusing input generation raw 2}. 


%
%
%
\begin{algorithm}[t]
\caption{BOSS Algorithm}
\textbf{Input}: $\mathbf{z}$, $p(.\:;\theta)$, $g$, $\mathbf{x}_{\textrm{d}}$, $p_{\textrm{d}}$, $\delta_{\textrm{c}}$, $\delta_{\textrm{s}}$ \\
\textbf{Output}: $\mathbf{x}$\\
\vspace{1mm}
{ 1:}  \textbf{Initialize} $\mathbf{x}$, $\phi$, $\lambda$ \\
\vspace{1mm}
{ 2:}  \textbf{while} $D(p(\mathbf{x}\:;\theta),p_{\textrm{d}}) \geq \delta_{\textrm{c}}$ \textbf{or} $d(\mathbf{x},\mathbf{x}_{\textrm{d}}) \geq \delta_{\textrm{s}}$ \\ 
\vspace{1mm}
{ 3:} \hspace{2mm}  \textbf{obtain} $\phi$ as the minimizer of \eqref{eqn: main wrt phi} with $\lambda$\\
\vspace{1mm}
{ 4:} \hspace{2mm} $\mathbf{x} = g\big(\mathbf{z}\:;\phi\big)$\\
\vspace{1mm}
{ 5:} \hspace{2mm} \textbf{update} $\lambda$ using \eqref{eqn: lablda h update}\\
\vspace{1mm}
{ 6:} \textbf{return} $\mathbf{x}$\\
\vspace{-4mm}
\label{alg:confusing input generation raw 2}
\end{algorithm}

\section{Experimental results}
\label{sec:subsubhead}

We show results for Targeted attacks which we call BOSS-T. The desired distribution is selected such that $p_{\textrm{d}}(l) = 1$ if $l$ is the target entry and $0$ otherwise. Second, Confidence reduction examples, which we dub BOSS-C. Let the true label of $\mathbf{x}_{\textrm{d}}$ be $f^*$ and the desired confidence be $c_{\textrm{d}}$, then $p_{\textrm{d}}(l) = c_{\textrm{d}}$ if $l=f^*$, and $(1-c_{\textrm{d}})/(M-1)$ otherwise. In addition to the samples from BOSS-T and BOSS-C, we also show instances of boundary adversarial examples. In this case, 
a value of $0.5$ is assigned to the two class labels on both sides of the boundary.


We use $D$ as the JS distance to compare PMFs (desired and actual) and the SSIM index $I$ as a measure of similarity between examples. We define $\sigma_\textrm{JS}$ and $\sigma_\textrm{s}$ as the average of $D$ and $I$, respectively, over the set of observations $\mathcal{X}$. In addition to the aforementioned metrics, for BOSS-T, we utilize the \emph{attack success rate} $\alpha =: N_{\textrm{s}} / |\mathcal{X}|$, where $N_{\textrm{s}}$ is the number of times a generated adversarial sample is classified as the predefined target label. \textcolor{black}{Further, we use $\sigma_2$ and $\sigma_{\infty}$ to denote the average $l_2$ and $l_{\infty}$, respectively.} For BOSS-C, we compute  $\sigma_{\textrm{con}}$, defined as the average confidence level of prediction of the true label over the set of interest $\mathcal{X}$.\par

\textcolor{black}{The random vector $\mathbf{z}$ of dimension $Q=100$ is generated from a uniform distribution over the interval $[0,1]$, \textcolor{black}{and 80 repeated samples are used for training}. The initial loss weights are chosen as $\lambda^{0} = \lambda^{0}_{v} = 0.001$. The parameters are updated using the ADAM optimizer \cite{kingma2014adam} with initial step size $0.025$. The details of the pre-trained classifiers and the generative networks, and our code are available online\footnote{https://github.com/ialkhouri/BOSS}}.  




\begin{table}
\caption{{\small{BOSS-C and NewtonFool with $c_{\textrm{d}} = 0.6$, $\delta_{\textrm{c}} = 0.2$, and $\delta_{\textrm{s}} = 0.85$ for the MNIST dataset.}}}
\vspace{-0.3cm}
\label{confidence reducntion atatck performance}\centering
 \scalebox{0.85}{\begin{tabular}{||c c c c c||} 
 \hline
 Environment & $\textrm{CA}(\%)$ & $\sigma_{\textrm{con}}(\%)$ & $\sigma_{\textrm{s}}(\%)$ & $\sigma_{\textrm{JS}}$ \\[0.5ex] 
 \hline\hline
 Model & 98.1  & 99.1 & 100 &  0.03  \\
 \hline
 Model+BOSS-C  &  98.1 & 66.73  & 87.89 & 0.19  \\
  \hline
 Model+NewtonFool \cite{jang2017objective}  &  75.5 & 67.85  & 97.02 & 0.48  \\
 \hline
 
\end{tabular}}
\end{table}

\begin{table*}
\caption{{\small{BOSS-T attack overall comparison with state-of-the-art attack methods using the CIFAR-10 dataset. For targets, all labels other than the predicted ones are considered.}}}
\label{targeted attacks comparison}\centering
\vspace{-0.3cm}
\scalebox{0.95}{\begin{tabular}{||c c c c c c c||} 
\hline
Attack &  $\alpha (\%)$ & $\sigma_2$ & $\sigma_{\infty}$ & $\sigma_{s}$ & Average Adversarial Confidence & Average Run Time (sec)  \\ 
\hline\hline
CW-$l_2$ ($\kappa=0$) \cite{carlini2017towards}  & 99.55 &  0.4195 & 0.0519 & 0.9966 & 0.4371 & 97.0982 \\
\hline
CW-$l_2$ ($\kappa=10$) \cite{carlini2017towards}  & 95.2381 &  0.7755 & 0.0907 & 0.9889 & 0.9964 & 95.5681 \\
\hline
CW-$l_{\infty}$ ($\kappa=0$) \cite{carlini2017towards}  & 99.55 &  1.0761 & 0.1188 & 0.9780 & 0.7126 & 0.1744 \\
\hline
CW-$l_{\infty}$ ($\kappa=10$) \cite{carlini2017towards}  & 96.82 &  1.91 & 0.1676 & 0.9999 & 0.9966 & 0.8912 \\
\hline
EAD (EN decision) \cite{chen2018ead}  & 100 &  0.4419 & 0.0885 & 0.9961 & 0.3951 & 141.971 \\
\hline
EAD ($l_1$ decision) \cite{chen2018ead}  & 100 &  0.5618 & 0.173 & 0.9938 & 0.3754 & 142.394 \\

\hline\hline
BOSS-T (MSE)  & 100 &  1.1589 & 0.1046 & 0.9818 & 0.9879 & 16.5554 \\
\hline
BOSS-T (Huber)  & 100 &  1.1454 & 0.1075 & 0.9799 & 0.9806 & 21.6079 \\
\hline
BOSS-T (log cosh)  & 100 &  1.1055 & 0.1042 & 0.9828 & 0.9785 & 21.1732 \\
\hline

\end{tabular}}
\end{table*}

\begin{figure}[t]
    \centering
    \includegraphics[width=8.5cm]{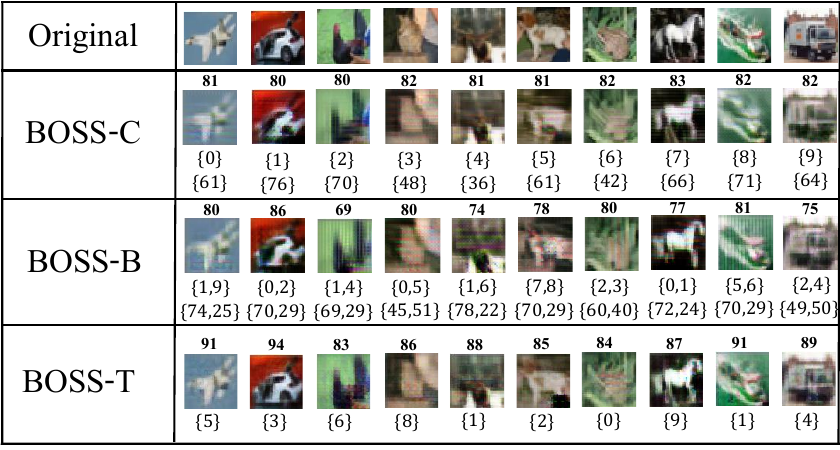} 
\caption{\small{Samples from each class of the CIFAR-10 dataset (columns). The first row shows the original examples. Rows 2-4 represent the synthesized images for BOSS-C, BOSS-B, and BOSS-T, respectively. The rounded percentage values of the confidence level, $c$, of the BOSS-C samples are placed at the bottom of each image along with the predicted label. For BOSS-B, the first pair at the bottom of every image represents the highest two predicted labels along with their rounded classification scores (second pair). Predicted labels are placed at the bottom of each BOSS-T example. The percentage of the rounded similarity measure ($I$) is placed on top of each generated example.}}
    \label{fig: cifar samples}
    \vspace{-0.5cm}
\end{figure}

\begin{figure}[t]
    \centering
    \includegraphics[width=8.5cm]{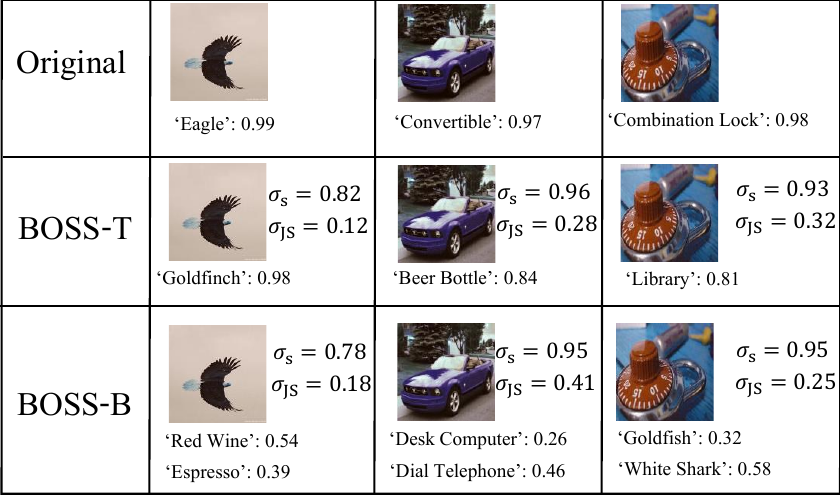} 
\caption{\small{Clean samples from some classes of the ImageNet dataset on VGG16. The first row shows the original examples. Rows 2 and 3 show the synthesized images for BOSS-T and BOSS-B, respectively. The labels are given at the bottom of each example along with the classification confidence. The SSIM and JS measures are reported on the right of each BOSS image.}}

    \label{fig: imageNet samples}
    \vspace{-0.5cm}
\end{figure}


For BOSS-C, Table \ref{confidence reducntion atatck performance} presents the results for $\sigma_{\textrm{con}}$, $\sigma_{\textrm{s}}$, and $\sigma_{\textrm{JS}}$. \textcolor{black}{For NewtonFool, we use $50$ iterations and set the \textcolor{black}{small perturbations parameter as} $\eta = 0.01$}. For an average confidence of $\sigma_{\textrm{con}}\approx 67\%$, BOSS-C (with $c_\textrm{d}=0.6$) and NewtonFool are successful in reducing the average confidence of the model from \textcolor{black}{the original value $\sigma_{\textrm{con}}=99.1\%$}. This is accomplished with very high level of similarity measure of $\sigma_{\textrm{s}} \approx 88\%$ and $\sigma_{\textrm{s}} \approx 97\%$ for BOSS-C and NewtonFool, respectively. While NewtonFool attack produces examples with higher $\sigma_{\textrm{s}}$, it fails to maintain the classification accuracy $\textrm{CA}$ which drops from \textcolor{black}{$98.1\%$} to $75.5\%$, and yields a large distance $\sigma_{\textrm{JS}}$ from the desired PMF.\par


The results for BOSS-T are presented in Table \ref{targeted attacks comparison} and compared to the state-of-the-art CW \cite{carlini2017towards} and elastic nets attacks (EAD) \cite{Chen_Sharma_Zhang_Yi_Hsieh_2018}. \textcolor{black}{
We choose these baselines since their formulations admit any differentiable loss function, unlike the well-known Projected Gradient Descent method \cite{madry2017towards} where the distance between $\mathbf{x}$ and $\mathbf{x}_{\textrm{d}}$ is limited to the $l_p$ norm. 
} For each testing example, all labels other than the predicted one are used as targets. Results of the average adversarial confidence and average run time are reported for each case in the last two columns. The parameters for CW and EAD on CIFAR10 are selected from the reported parameters in their respective papers. It is important to note that both of these methods apply their attacks based on the pre-softmax output (sometimes called logits), and hence, cannot specify an exact $p_{\textrm{d}}$. However, in the CW formulation, the parameter $\kappa$ was introduced to represent the desired logit value to achieve higher adversarial confidence. While setting $\kappa=0$ in the CW attack returns the best result in terms of imperceptibility, it does not yield the best adversarial confidence. Therefore, we report results for $\kappa = 10$, which yields a better tradeoff between both measures. Furthermore, we implement BOSS-T with different surrogate loss functions in $\mathbf{x}$ and $\mathbf{x}_{\textrm{d}}$.\par

\textcolor{black}{While some variant of EAD and CW achieve a relatively lower imperceptibility (as seen from $\sigma_2$ and $\sigma_{\infty}$), in terms of adversarial confidence, all variants of BOSS-T return the best results. CW, with $\kappa=10$, reports similar adversarial confidence, but the attack success ratio does not achieve 100\%, and requires 5 times the run time for $l_2$ and nearly 50\% increase in imperceptibility (presented in $\sigma_2$ and $\sigma_{\infty}$) are observed.}\par


\textcolor{black}{Figures \ref{fig:motiv}, \ref{fig: cifar samples}, and \ref{fig: imageNet samples} show BOSS-C, BOSS-T, and BOSS-B samples. As observed, BOSS is successful in generating adversarial examples given the desired input specification, represented the original images in the first row, and output specification as represented in the corresponding application.}


\section{Conclusion}
We introduced BOSS, a framework for one-shot synthesis of adversarial samples that satisfy input and output specifications for pre-trained classifiers. We formulated the BOSS problem and proved that the problem is \textbf{NP-Complete}. We developed an approximate solution using generative networks and surrogate loss functions. The flexibility of BOSS is demonstrated through various applications, including synthesis of boundary examples, targeted attacks, and reduction of confidence samples. A set of experiments verify that BOSS, in general, performs on par with state-of-the-art methods and generates the highest adversarial confidence examples.




\small{
\bibliographystyle{IEEEbib}
\bibliography{refs}

\begin{thebibliography}{10}

\bibitem{wang2020generalizing}
Yaqing Wang, Quanming Yao, James~T Kwok, and Lionel~M Ni,
\newblock ``Generalizing from a few examples: A survey on few-shot learning,''
\newblock {\em ACM Computing Surveys (CSUR)}, vol. 53, no. 3, pp. 1--34, 2020.

\bibitem{machado2020adversarial}
Gabriel~Resende Machado, Eug{\^e}nio Silva, and Ronaldo~Ribeiro Goldschmidt,
\newblock ``Adversarial machine learning in image classification: A survey
  towards the defender's perspective,''
\newblock {\em arXiv preprint arXiv:2009.03728}, 2020.

\bibitem{papernot2016limitations}
Nicolas Papernot, Patrick McDaniel, Somesh Jha, Matt Fredrikson, Z~Berkay
  Celik, and Ananthram Swami,
\newblock ``The limitations of deep learning in adversarial settings,''
\newblock in {\em IEEE European Symposium on Security and Privacy (EuroS\&P)},
  2016, pp. 372--387.

\bibitem{pmlr-v119-stutz20a}
David Stutz, Matthias Hein, and Bernt Schiele,
\newblock ``Confidence-calibrated adversarial training: Generalizing to unseen
  attacks,''
\newblock in {\em Proceedings of the 37th International Conference on Machine
  Learning}, Hal~Daumé III and Aarti Singh, Eds. 13--18 Jul 2020, vol. 119 of
  {\em Proceedings of Machine Learning Research}, pp. 9155--9166, PMLR.

\bibitem{lecun2010mnist}
Yann LeCun, Corinna Cortes, and CJ~Burges,
\newblock ``Mnist handwritten digit database,''
\newblock {\em ATT Labs [Online]. Available: http://yann.lecun.com/exdb/mnist},
  vol. 2, 2010.

\bibitem{DBLP:journals/corr/abs-1708-07747}
Han Xiao, Kashif Rasul, and Roland Vollgraf,
\newblock ``{Fashion-MNIST}: a novel image dataset for benchmarking machine
  learning algorithms,''
\newblock {\em CoRR}, vol. abs/1708.07747, 2017.

\bibitem{krizhevsky2009learning}
Alex Krizhevsky et~al.,
\newblock ``Learning multiple layers of features from tiny images,''
\newblock 2009.

\bibitem{Stallkamp2012}
J.~Stallkamp, M.~Schlipsing, J.~Salmen, and C.~Igel,
\newblock ``Man vs. computer: Benchmarking machine learning algorithms for
  traffic sign recognition,''
\newblock {\em Neural Networks}, , no. 0, pp. --, 2012.

\bibitem{5206848}
Jia Deng, Wei Dong, Richard Socher, Li-Jia Li, Kai Li, and Li~Fei-Fei,
\newblock ``Imagenet: A large-scale hierarchical image database,''
\newblock in {\em 2009 IEEE Conference on Computer Vision and Pattern
  Recognition}, 2009, pp. 248--255.

\bibitem{goodfellow2014generative}
Ian~J Goodfellow, Jean Pouget-Abadie, Mehdi Mirza, Bing Xu, David Warde-Farley,
  Sherjil Ozair, Aaron~C Courville, and Yoshua Bengio,
\newblock ``Generative adversarial nets,''
\newblock in {\em NIPS}, 2014.

\bibitem{carlini2017towards}
Nicholas Carlini and David Wagner,
\newblock ``Towards evaluating the robustness of neural networks,''
\newblock in {\em IEEE Symposium on Security and Privacy}, 2017, pp. 39--57.

\bibitem{szegedy2013intriguing}
Christian Szegedy, Wojciech Zaremba, Ilya Sutskever, Joan Bruna, Dumitru Erhan,
  Ian Goodfellow, and Rob Fergus,
\newblock ``Intriguing properties of neural networks,''
\newblock {\em preprint arXiv:1312.6199}, 2013.

\bibitem{moosavi2016deepfool}
Seyed-Mohsen Moosavi-Dezfooli, Alhussein Fawzi, and Pascal Frossard,
\newblock ``Deepfool: a simple and accurate method to fool deep neural
  networks,''
\newblock in {\em Proceedings of the IEEE conference on computer vision and
  pattern recognition}, 2016, pp. 2574--2582.

\bibitem{croce2020minimally}
Francesco Croce and Matthias Hein,
\newblock ``Minimally distorted adversarial examples with a fast adaptive
  boundary attack,''
\newblock in {\em International Conference on Machine Learning}. PMLR, 2020,
  pp. 2196--2205.

\bibitem{jang2017objective}
Uyeong Jang, Xi~Wu, and Somesh Jha,
\newblock ``Objective metrics and gradient descent algorithms for adversarial
  examples in machine learning,''
\newblock in {\em Proceedings of the 33rd Annual Computer Security Applications
  Conference}, 2017, pp. 262--277.

\bibitem{riedmiller1993direct}
Martin Riedmiller and Heinrich Braun,
\newblock ``A direct adaptive method for faster backpropagation learning: The
  rprop algorithm,''
\newblock in {\em IEEE International Conference on Neural Networks}, 1993, pp.
  586--591.

\bibitem{kurakin2016adversarial}
Alexey Kurakin, Ian Goodfellow, and Samy Bengio,
\newblock ``Adversarial machine learning at scale,''
\newblock {\em arXiv preprint arXiv:1611.01236}, 2016.

\bibitem{chen2018ead}
Pin-Yu Chen, Yash Sharma, Huan Zhang, Jinfeng Yi, and Cho-Jui Hsieh,
\newblock ``Ead: elastic-net attacks to deep neural networks via adversarial
  examples,''
\newblock in {\em Proceedings of the AAAI Conference on Artificial
  Intelligence}, 2018, vol.~32.

\bibitem{goodfellow2016deep}
Ian Goodfellow, Yoshua Bengio, Aaron Courville, and Yoshua Bengio,
\newblock {\em Deep learning}, vol.~1,
\newblock MIT press Cambridge, 2016.

\bibitem{lin1991divergence}
Jianhua Lin,
\newblock ``Divergence measures based on the shannon entropy,''
\newblock {\em IEEE Transactions on Information theory}, vol. 37, no. 1, pp.
  145--151, 1991.

\bibitem{laidlaw2020perceptual}
Cassidy Laidlaw, Sahil Singla, and Soheil Feizi,
\newblock ``Perceptual adversarial robustness: Defense against unseen threat
  models,''
\newblock in {\em International Conference on Learning Representations}, 2020.

\bibitem{wang2004image}
Zhou Wang, Alan~C Bovik, Hamid~R Sheikh, and Eero~P Simoncelli,
\newblock ``Image quality assessment: from error visibility to structural
  similarity,''
\newblock {\em IEEE Transactions on Image Processing}, vol. 13, no. 4, pp.
  600--612, 2004.

\bibitem{kingma2014adam}
Diederik~P Kingma and Jimmy Ba,
\newblock ``Adam: A method for stochastic optimization,''
\newblock {\em arXiv preprint arXiv:1412.6980}, 2014.

\bibitem{Chen_Sharma_Zhang_Yi_Hsieh_2018}
Pin-Yu Chen, Yash Sharma, Huan Zhang, Jinfeng Yi, and Cho-Jui Hsieh,
\newblock ``Ead: Elastic-net attacks to deep neural networks via adversarial
  examples,''
\newblock {\em Proceedings of the AAAI Conference on Artificial Intelligence},
  vol. 32, no. 1, Apr. 2018.

\bibitem{madry2017towards}
Aleksander Madry, Aleksandar Makelov, Ludwig Schmidt, Dimitris Tsipras, and
  Adrian Vladu,
\newblock ``Towards deep learning models resistant to adversarial attacks,''
\newblock {\em arXiv preprint arXiv:1706.06083}, 2017.

\end{thebibliography}
}
\end{document}